\documentclass[12pt]{article}
\usepackage{amsmath}
\usepackage{times}
\usepackage{graphicx}
\usepackage{color}
\usepackage{multirow}
\usepackage[natbibapa]{apacite}
\usepackage{rotating}
\usepackage{bbm}
\usepackage{latexsym}
\usepackage{times}
\usepackage{epsfig}
\usepackage{graphicx}
\usepackage{amsmath}
\usepackage{amssymb}

\usepackage{amssymb}
\usepackage{soul}
\usepackage{capt-of}
\usepackage{adjustbox}
\usepackage[utf8]{inputenc} 
\usepackage{latexsym}
\usepackage{amsthm}
\usepackage{cite}
\usepackage[font=small,labelfont=bf]{caption} 
\usepackage[subrefformat=parens]{subcaption}
\usepackage{epsfig}
\usepackage{subcaption}
\newtheorem{lemma}{Lemma}
\usepackage[pagebackref=true,breaklinks=true,colorlinks,bookmarks=false]{hyperref}

\newcommand{\captionfonts}{\normalsize}

\makeatletter  
\long\def\@makecaption#1#2{%
  \vskip\abovecaptionskip
  \sbox\@tempboxa{{\captionfonts #1: #2}}%
  \ifdim \wd\@tempboxa >\hsize
    {\captionfonts #1: #2\par}
  \else
    \hbox to\hsize{\hfil\box\@tempboxa\hfil}%
  \fi
  \vskip\belowcaptionskip}
\makeatother   

\begin{document}
\hspace{13.9cm}1

\ \vspace{20mm}\\

\begin{center}
    {\LARGE Unsupervised Discovery, Control, and Disentanglement of Semantic Attributes with Applications to Anomaly Detection}
\end{center}

\ \\
{\bf \large William Paul$^{\displaystyle 1}$, I-Jeng Wang$^{\displaystyle 1}$, Fady Alajaji$^{\displaystyle 2}$, Philippe Burlina$^{\displaystyle 1, 3}$}\\
{$^{\displaystyle 1}$The Johns Hopkins University Applied Physics Laboratory\\
11100 Johns Hopkins Rd,
Laurel, MD 20723, USA\\
{\tt\small {firstname.lastname}@jhuapl.edu}}\\
{$^{\displaystyle 2}$Department of Mathematics and Statistics\\ Queens University, ON K7L 3N6, Canada \\
{\tt\small {fa}@queensu.ca}}\\
{$^{\displaystyle 3}$Department of Computer Science\\ Johns Hopkins University\\
3400 N. Charles Street
Baltimore, MD 21218\\}

\thispagestyle{empty}
\markboth{}{NC instructions}
\ \vspace{-0mm}\\
\begin{center} {\bf Abstract} \end{center}
Our work focuses 
on unsupervised and generative methods that address the following goals: 
(a) learning unsupervised generative representations that discover latent factors controlling image semantic attributes, (b) studying how this ability to control attributes formally relates to the issue of latent factor disentanglement, clarifying related but dissimilar concepts that had been confounded in the past, and (c) developing anomaly detection methods that leverage representations learned in (a). 

For (a) we propose a network architecture that exploits the combination of multiscale generative models with mutual information (MI) maximization.
For (b), we derive an analytical result (Lemma 1) that brings clarity to two related but distinct concepts: the ability of generative networks to control semantic attributes of images they generate, resulting from MI maximization, and the ability to disentangle latent space representations, obtained via total correlation minimization. More specifically, we demonstrate that maximizing semantic attribute control encourages disentanglement of latent factors.
Using Lemma 1 and adopting MI in our loss function, we then show empirically that, for image generation tasks, the proposed approach exhibits superior performance as measured in the quality and disentanglement trade space, when compared to other state of the art methods, with quality assessed via the Fr\'echet Inception Distance (FID), and disentanglement via mutual information gap. For (c), we design several systems for anomaly detection exploiting representations learned in (a), and  demonstrate their performance benefits when compared to state-of-the-art generative and discriminative algorithms.

The above contributions in representation learning have potential applications in addressing other important problems in computer vision, such as bias and privacy in AI. 

\section{Introduction}

\noindent
{\bf Motivations and Goals} 

There has been unsurpassed success in the application of deep learning (DL) in several areas of visual analysis, computer vision, natural language processing and medical imaging~\citep{litjens2017survey}. The transformative contribution of DL to artificial intelligence (AI), principally in discriminative models for supervised learning, has mostly hinged on the availability of large training datasets such as ImageNet \citep{ILSVRC15}. Open problems still remain in DL, especially in unsupervised learning,  inference on out-of-training-distribution test samples, domain shift, and also in discriminative tasks where data is not easily obtained or when manual labeling is impractical or prohibitively onerous. Generative models -- with their ability to generate data similar to a given dataset and efficient representations of this data -- may assist in addressing some of these challenges. 

Considering these challenges, the ability to generate images with good quality and diversity at high resolution and to allow the unsupervised discovery and control of individual semantic images attributes via latent space factors are of paramount importance, and are the main aims motivating our study. These goals are also extended here to applying these learned representations to the practical task of anomaly detection.

Generative methods -- broadly speaking -- learn to sample from the underlying training data distribution so as to generate novel, fake samples that are visually or statistically indistinguishable from the underlying training data. A simple taxonomy of generative models includes: generative adversarial networks (GANs) \citep[see, e.g.,][]{goodfellow2014generative, salimans2016improved}, autoencoders/variational autoencoders (VAEs) \citep{kingma2013auto}, and, used to a lesser extent in comparison to the aforementioned ones, generative autoregressive models in~\citep{oord2016wavenet}, invertible flow based latent vector models in~\citep{kingma2018glow}, or hybrids of the above models as in~\citep{grover2018flow}.

\medskip\noindent
{\bf Prior Work}

Prior and recent research in generative models addressing some of our aims and inspiring 
our work are organized below along the following areas:

\medskip
\noindent
{\bf \em High Resolution GAN Approaches} 

A goal of GANs has been to achieve good quality image generation for high resolution images. Despite sustained research in GANs, only relatively recently have GANs been able to generate images at somewhat high resolutions (greater than 256 by 256 pixels). Examples of methods that achieve such results include ProGAN \citep{karras2017progressive}, BigGAN \citep{brock2019biggan}, and COCO-GAN \citep{lin2019cocogan}. For example, BigGAN relied on SAGAN \citep{zhang2019sagan} as a baseline (self attention GANs) and used large batch sizes to improve performance (multiplying batch size by a factor of 8 leads to over 40\% increase in inception score over other state of the art algorithms). That study also noted that larger networks had a comparable positive effect, and so did the usage of the {\em truncation} trick (i.e., for the generator, sampling from a standard normal distribution in training while sampling instead from a truncated normal distribution in inference, where samples that are above a certain threshold are re-sampled). Truncating with a lower threshold allowed control of the trade-off between higher fidelity and lower diversity. 

Unfortunately, while many best of breed generative approaches made progress in terms of visual quality and high resolutions, these methods cannot be directly used for semantic attribute control, which is one of our goals in this paper.

\medskip
\noindent
{\bf \em Style Transfer} 

 StyleGAN \citep{Karras2019stylegan} has been very successful at addressing the generation of high dimensional images (1024 by 1024) and is an extension of ProGAN, which was based on progressively growing the encoder and decoder/discriminator in GAN networks. Some of the specific novel features in ~\citep{Karras2019stylegan} consisted of injecting noise at every scale resolution of the decoder, and of using a fully connected (FC) network that mapped a latent vector $Z$ into a 512-length intermediate  latent vector $W$. This so called 'style' vector had incidence on the generation of images throughout multiple scales and was used to influence some attributes of the image (at the low scale coarse attributes like skin tone and at the higher scale fine attributes like hair). 
 An updated version of StyleGAN \citep{Karras2019stylegan2}, StyleGAN2, introduced architectural improvements, along with an improved back projection method for mapping images to latent spaces.

While both methods successfully allowed for using multiple style vectors at different scales as a natural extension,  the resulting images may have included undesirable attributes from the original base images. These methods therefore did not allow for directly controlling specific semantic image attributes and, consequently, unsupervised discovery of consistent semantic attributes. In contrast, the approach we seek strives for the discovery, control and isolation of desirable attributes from spurious attributes for complex imagery.

\medskip\noindent
{\bf \em Information Theoretic Approaches} 

In~\citep{chen2016infogan}, InfoGAN  made use for the first time of the principle of maximizing the mutual information $\mathbb{I}(C; \hat{X})$  between a semantic latent vector $C$ and the generated image $\hat{X}$, where $C$ is a semantic component of the latent vector representation $Z=(Z', C)$, with $Z'$ being a noise vector. This process was originally set up to achieve disentanglement in latent factors between different scales. Experiments  exemplified various degrees of agreement between  such factors and semantic attributes. However, as our study should demonstrate, the mutual information principle actually promotes the control of attributes in images, but true disentanglement is not always achieved via this maximization of mutual information. Indeed, the results in \citep{chen2016infogan} suggested various degree of success and  consistency between the ability of the network  to actually control  and disentangle.

Instead, in our view, the concept of disentanglement was appropriately defined in \citep{chen2018isolating} as the process of finding latent factors that satisfy minimum total correlation, defined by the Kullback-Leibler divergence between the vector's joint distribution and the product of its marginal components. Despite this contribution, the  concepts of control and disentanglement have remained used interchangeably and are often confused with each other in recent literature. In this study, we adopt this definition and help bring order by formally demonstrating how this concept relates to the concept of control and maximization of mutual information. 

\medskip
\noindent
{\bf \em Generative Methods for Anomaly Detection} 

Our goal (c) strives to demonstrate the utility of disentangled representations for an important machine learning task, specifically, anomaly detection.
Related to this is past work that have used DL discriminative and generative model representation learning for anomaly detection. While DL-based anomaly detection schemes initially used discriminative representations -- e.g.~\citep{erfani2016high} -- where deep belief networks were combined with statistical methods, more recent methods that made use of generative representation learning include~\citep{akcay2018ganomaly,schlegl2017unsupervised,deecke2018anomaly,efficient,liu2018generative,gray2018coupled,kimura2018semi,abay2018maneuver,naphade20182018,lai2018industrial,bergmann2018improving,jain2018imagining}. Generally these exploit two main strategies: i) using cyclic reconstruction error in the image space as an anomaly detection metric or ii)  directly using the GAN discriminator. These two techniques were employed for pixel based anomaly detection in ~\citep{schlegl2017unsupervised} and for image based anomaly detection in \citep{efficient} and~\citep{deecke2018anomaly}. Another approach used metrics of reconstruction in latent/code space and is embodied in the work of~\citep{akcay2018ganomaly} 
The method in ~\citep{akcay2018ganomaly} was later extended to include skip connections~\citep{akccay2019skip}. However those methods did not probe the use of generative models that had ability for large resolution image generation and allowed disentanglement, as is done here.
Recently, \citep{burlina2019s} found that most generative approaches fell short of using discriminative embeddings. We consequently focus on comparing to these discriminative models as the state of the art.

\medskip\noindent
{\bf Novel contributions} 

When compared to prior work, the novel contributions of this work are as follows: 
\begin{enumerate}
\item We consider the disentanglement of image attributes and the unsupervised discovery of such attributes,\footnote{Concurrent work we became aware of after first submitting this work in arXiv includes~\citep{shen2020cffac, weili2020semistyle, tewari2020stylerig,harkonen2020ganspace}} proposing a novel approach relying on the combination of mutual information maximization with multiscale GANs. 
\item We bring clarity to the concepts of semantic control and disentanglement, demonstrating that there is a connection between mutual information maximization and total correlation minimization, i.e., between the concepts of attribute control and disentanglement: we demonstrate in Lemma \ref{lemma_proof} that maximizing semantic attribute control encourages the minimizing of entanglement for latent factors. 
\item We show empirically that the proposed approach results in high resolution generation with the ability for unsupervised discovery of latent codes that help control specific semantic image attributes. 
\item We develop several methods using the proposed generative architectures, used for representation learning, for the end task of anomaly detection. We  then demonstrate empirically the resulting performance benefits when compared to other state of the art discriminative methods.
\end{enumerate}

\section{Approach}

The next subsections present our approach including: details on representation learning via our proposed architecture (InfoStyleGAN) in subsection \ref{infostylegan} and the loss function design and analysis in subsection \ref{loss}.

\subsection{Definitions of Attribute Control, Disentanglement, Architecture and Loss Functions}
\label{infostylegan}

We now discuss the concepts of attribute control and disentanglement. Consider first the architecture used here which is depicted in Figure~\ref{fig:TCICSG} and borrows some of the GAN components in StyleGAN \citep{Karras2019stylegan}, including a multiscale generator $\hat{X}=\mathcal{G}(Z)$ and a discriminator $\mathcal{D}$.
\medskip 

\begin{figure*}[h!]
    \centering
    \includegraphics[width=14cm]{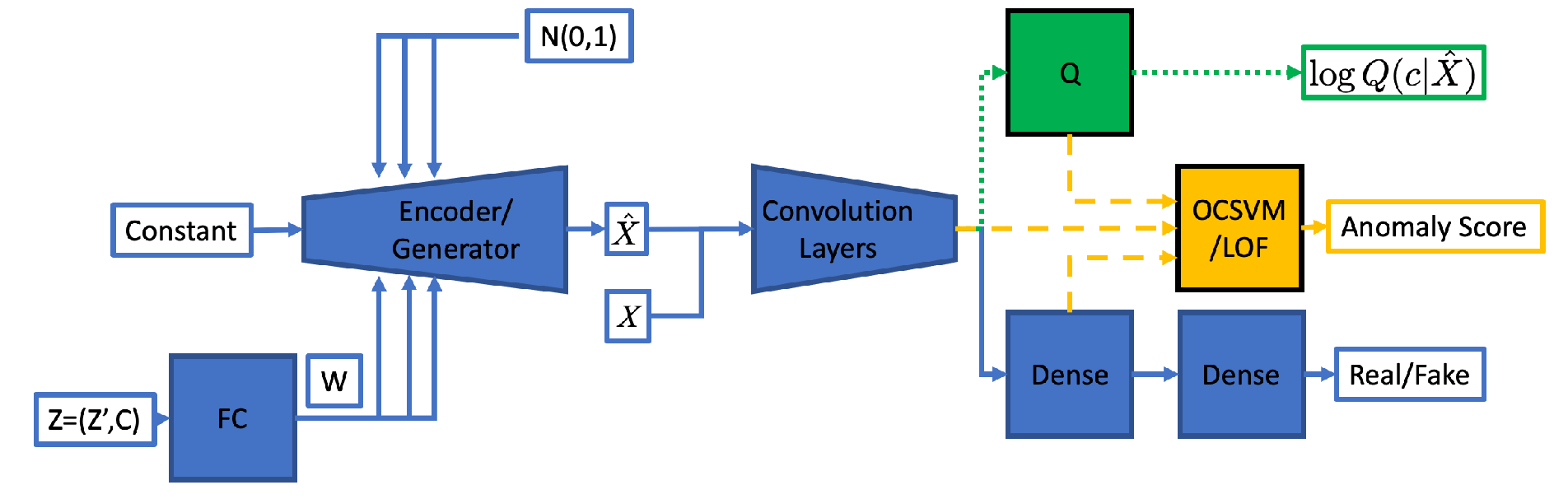}
    \caption{Our proposed multi-scale generator and discriminator architecture: The latent code in $Z$ is split into noise terms $Z'$ and semantically relevant variables $C$. Sample $(z', c)$ is fed into a mapping network and a mutual information maximizing loss is used between latent and output generated image, combined with a conditional loss and traditional GAN adversarial loss. The pathways for the information $Q$ auxiliary network are shown in dotted green while the pathways for anomaly detection via representation and OCSVM and LOF are shown in dashed yellow.}
    \label{fig:TCICSG}
\end{figure*}{}
Since we seek to learn latent space representations that relate to generated image attributes, for attribute discovery and control, the latent vector $Z=(Z',C)$ is decomposed into a standard Gaussian noise vector $Z'$ and a latent vector component $C$ (henceforth called the latent factors) with distribution $p(C)$, where $C$ and $Z'$ are independent of each other. 

We define the concept of disentanglement as the minimization of the total correlation between the different {\em latent factors} $C$, while the faculty for discovery and control of {\em semantic image attributes} is defined via the {\em maximization of  mutual information} between the latent factors and the generated image. The latter concept is explicitly used in our considered loss function. We will demonstrate analytically that the former concept of total correlation minimization is implied -- under certain conditions -- from the latter principle of mutual information maximization.

\noindent
{\bf Control via maximization of mutual information (MI):} 
The maximization of the mutual information $I(C; \mathcal{G}(Z', C))$ between the semantic vector $C$ and the observation $\hat{X}=\mathcal{G}(Z',C)$ is used as a means of discovering the factors of variations in images and forcing the coupling the vector $C$ to the different factors of variations in the images $X$ in the dataset. Since MI computation is complicated by the fact it entails knowing the posterior $p(C|\hat{X})$, one can instead, as in InfoGAN in ~\citep{chen2016infogan}, employ an auxiliary distribution $\mathcal{Q}(C|\hat{X})$ that approximates this posterior and that can be selected to maximize the resulting MI measure.

\noindent
{\bf Disentanglement via minimization of total correlation:} Although \citep{chen2016infogan} introduced an auxiliary loss to maximize the MI between $C = (C_1, C_2, \dots, C_L)$, where each $C_i$ can be governed by a different distribution, and $\hat{X} = \mathcal{G}(Z',C)$, where $Z'$ is an additional noise vector, there is no explicit objective for controlling a diverse set of image attributes. For example, even with the independence implementation on the prior $p(C) = \prod_i p(C_i)$, in \citep{chen2016infogan}, all semantically relevant variables could only seemingly affect the skin in faces, with different variables focusing on skin tone, skin texture, glare on the skin, etc. In effect, despite being sampled as independent, the effect between individual variables on the image are highly correlated. Consequently, a desirable end goal for disentanglement is that knowledge of one latent factor from the image does not affect the knowledge of other latent factors; i.e., having $p(C | \hat{X}=\hat{x}) = \prod_i p(C_i | \hat{X}=\hat{x})$ or conditional independence of the true posterior of the latent factors given a realization $\hat{X}=\hat{x}$ of the generated image. 
This is equivalent to having $\mathbb{TC}(C|\hat{X}=\hat{x}) = 0$, where $\mathbb{TC}(\cdot|\hat{X}=\hat{x})$ is the total correlation \citep{chen2018isolating} given $\hat{X}=\hat{x}$, defined as the Kullback-Leibler (KL) divergence, denoted by $D_{KL}(\cdot \| \cdot)$, between the conditional joint distribution given $\hat{X}=\hat{x}$ and the product of the conditional marginal distributions given $\hat{X}=\hat{x}$: 
$$\mathbb{TC}(C| \hat{X}=\hat{x}) = D_{KL}\Big(p(C | \hat{X}=\hat{x}) \Big\| \prod_{i} p(C_i | \hat{X}=\hat{x})) \Big).$$

We will argue that the previous two concepts are connected to each other as shown in Lemma~\ref{lemma_proof}. Indeed, under certain assumptions, MI maximization constrains the total correlation. Therefore only one such constraint, namely the MI constraint, will be considered henceforth.

\subsection{Loss Function}
\label{loss}
The loss function is comprised of two parts: the adversarial loss $V(\mathcal{D},\mathcal{G})$ for the generator $\mathcal{G}$ and discriminator $\mathcal{D}$ as well as the lower bound $L_{\text{info}}(\mathcal{G}, Q)$ on the mutual information
\begin{align}
V(\mathcal{D},\mathcal{G}) &= \mathbb{E}_{C \sim p(C), \hat{X} \sim \mathcal{G}(Z', C)} (\log (1-\mathcal{D}(\hat{X}))) + \mathbb{E}_{X} (\log (\mathcal{D}(X))) \label{labeledloss} \\
L_{\text{info}}(\mathcal{G}, Q) &= \mathbb{E}_{C \sim p(C), \hat{X} \sim \mathcal{G}(Z', C)}(\log (Q(C | \hat{X})) - \log p(C))
\end{align}
where $\mathbb{E}(\cdot)$ denotes the expectation operator, $Q$ denotes the auxiliary network referenced earlier, and $0 \le \mathcal{D}(\cdot) \le 1$. 
Note that we abused the notation in $\hat{X} \sim \mathcal{G}(Z', C)$ to denote the conditional distribution $p(\hat{X}|C)$ given $C$ (this conditional distribution is intrinsically determined by the distribution of $Z'$ and the Gaussian noise maps).
The optimization problem consists then of determining the triplet $(\mathcal{D}, \mathcal{G}, Q)$ that achieves:
\begin{align}
&\min_{\mathcal{G}, Q} \max_{\mathcal{D}}\ V(\mathcal{D},\mathcal{G}) - \beta L_{\text{info}}(\mathcal{G},Q)  \label{L-overall}
\end{align}
where coefficient $\beta\ge 0$ is a hyper-parameter.  

Monte Carlo estimates of $V(\mathcal{D}, \mathcal{G})$ and $L_{\text{info}}(\mathcal{G},Q)$, $\Tilde{V}(\mathcal{D},Q)$ and $\Tilde{L}_{\text{info}}(\mathcal{G},Q)$ respectively, are used for realizing a tractable optimization. Using a batch size $B$, $\{z'(l), c'(l)\}_{l=1}^B$ is sampled from $Z$, as $Z$ is explicitly defined beforehand, and then fed into the generator to produce the fake image $\hat{x}(l) = \mathcal{G}(z'(l), c'(l))$ $B$ times, as well as the real images $x(l)$ being sampled $B$ times. The estimates for the losses are thus given by:
\begin{align}
&\Tilde{V}(\mathcal{D},\mathcal{G}) = \frac{1}{B} \sum_{l=1}^B (\log (1-\mathcal{D}(\hat{x}(l))) + \log (\mathcal{D}(x(l))),  \\
&\Tilde{L}_{\text{info}}(\mathcal{G},Q) = \frac{1}{B} \sum_{l=1}^B (\log (Q(c'(l) | \hat{x}(l))) - \log p(c'(l))). 
\end{align}

For nomenclature, we refer to InfoGAN as disabling styles and setting $\beta = 1$, StyleGAN as having styles and $\beta = 0$, InfoStyleGAN as having styles and $\beta = 1$, and InfoStyleGAN-Discrete as InfoStyleGAN with C only containing discrete variables.

\section{Relating Attribute Control and Disentaglement}
\label{decorrleationsection}
We next show that the choice of the mean field encoder, i.e., using $Q(C | \hat{X}) = \prod_i Q(C_i | \hat{X})$, for optimizing the mutual information (as used in \citep{chen2016infogan}) contributes to forcing the total correlation to zero. 
\begin{lemma}
\label{lemma_proof}
{\rm
Assume that each $C_i$ is discrete for $i=1,\ldots,L$ (hence $\mathbb{I}(C; \hat{X})$ is bounded from above)
and that $Q(C|\hat{X})=\prod_i Q(C_i|\hat{X})$. When $L_{\text{info}} \rightarrow \mathbb{I}(C;\hat{X})$,
then $$\mathbb{TC}(C| \hat{X}) \rightarrow 0$$ 
almost everywhere in $\hat{X}$, where $\mathbb{TC}(C| \hat{X})$
is the total correlation in $C$ given $\hat{X}$.
}
\end{lemma} 
\begin{proof}
First, recall the derivation of the InfoGAN objective from~\citep{chen2016infogan}:
\begin{align}
    \mathbb{I}(C; \hat{X}) &= - \mathbb{H}(C | \hat{X}) + \mathbb{H}(C)  \\
    &= \mathbb{E}_{\hat{X}}[\mathbb{E}_{C' \sim p(C|\hat{X})}( \log p(C' | \hat{X}))] + \mathbb{H}(C) \label{eq:lb} \\
    &= \mathbb{E}_{\hat{X}}[\underbrace{D_{KL}(p(C | \hat{X}) \| Q(C | \hat{X}))}_{\geq 0} + \mathbb{E}_{C' \sim p(C|\hat{X})}( \log Q(C' | \hat{X}))] + \mathbb{H}(C)  \\
    &\geq \mathbb{E}_{\hat{X}}[\mathbb{E}_{C' \sim p(C|\hat{X})}( \log Q(C' | \hat{X}))] + \mathbb{H}(C)  \label{eq:lb*} \\
    &= \mathbb{E}_{C \sim p(C), \hat{X} \sim \mathcal{G}(Z', C)}( \log Q(C | \hat{X})) + \mathbb{H}(C)  \nonumber \\
    &= L_{\text{info}} 
\end{align}
where $\mathbb{H}(\cdot)$ and $\mathbb{H}(\cdot|\cdot)$ denote entropy and conditional entropy, respectively. Note that $\mathbb{H}(C) = \sum_i \mathbb{H}(C_i)$ as we assume the prior to the generative model factorizes independently. 
Starting from \eqref{eq:lb}, we can decompose the logarithmic term for each individual $C_i$ to get:
\begin{align}
    \mathbb{I}(C; \hat{X}) &= \mathbb{E}_{\hat{X}}\left[\mathbb{E}_{C' \sim p(C|\hat{X})}\left( \log \frac{p(C' | \hat{X})}{\prod_i p(C'_i | \hat{X})}\right)\right] \nonumber \\ &\ \ \ \ + \sum_i \left(\mathbb{E}_{ \hat{X}}[\mathbb{E}_{C_i' \sim p(C_i|\hat{X})}( \log p(C'_i | \hat{X}))]+ \mathbb{H}(C_i)\right) \\
    &= \mathbb{E}_{C \sim p(C), \hat{X} \sim \mathcal{G}(Z', C)} (\mathbb{TC}(C|\hat{X})) + \sum_i \mathbb{I}(C_i; \hat{X}). \label{eq:mi}
\end{align}
Now, purely maximizing the mutual information with respect to all variables could also increase the total correlation of the posterior $p(C | \hat{X})$ of a fake image $\hat{X}$, implying that the factors given the image are more entangled, which is undesirable. Thus, we desire a low total correlation $\mathbb{TC}$ and high mutual information $\mathbb{I}$ between $\hat{X}$, the generated image, and each variable $C_i$ individually, which we argue that the original InfoGAN objective implicitly satisfies for discrete (finite-valued) variables. \\
Indeed, we can lower bound each individual $\mathbb{I}(C_i; \hat{X})$ term in \eqref{eq:mi} via the same method as in \eqref{eq:lb*}: we have
\begin{align}
    \mathbb{H}(C) &\geq \mathbb{I}(C; \hat{X}) \\
    &\geq \mathbb{I}(C; \hat{X}) - \mathbb{E}_{C \sim p(C), \hat{X} \sim \mathcal{G}(Z', C)} (\mathbb{TC}(C|\hat{X})) \label{eq-sub0} \\
    &= \sum_i \mathbb{I}(C_i; \hat{X}) \label{eq-sub1} \\
    &\geq \sum_i \left(\mathbb{E}_{\hat{X}}[\mathbb{E}_{   C_i \sim p(C_i| \hat{X})}( \log Q(C_i | \hat{X}))]  + \mathbb{H}(C_i)\right) \label{eq-sub2} \\
    &= \mathbb{E}_{C \sim p(C), \hat{X} \sim \mathcal{G}(Z', C)}( \log Q(C | \hat{X})) + \mathbb{H}(C) = L_\text{info}
\end{align}
where \eqref{eq-sub1} holds by \eqref{eq:mi}, \eqref{eq-sub2} follows from \eqref{eq:lb*} applied to each $\mathbb{I}(C_i; \hat{X})$ term. The equality before last is due to our assumption of a mean field encoder and the fact that $p(C) = \prod_i p(C_i)$. 

Thus, when $L_\text{info} \to \mathbb{I}(C; \hat{X})$, the two inequalities in \eqref{eq-sub0} and \eqref{eq-sub2} become tight, in turn implying that $Q(C|\hat{X}) \to p(C|\hat{X})$ and that $\mathbb{TC}(C| \hat{X}) \rightarrow 0$ almost everywhere in $\hat{X}$. 
\end{proof}

As explained earlier, while this lemma only applies when $C$ is a discrete random vector, it nevertheless provides useful insights regarding the connection between MI maximization and total correlation minimization, and between the concepts of attribute control and disentanglement, which had been conflated in past literature, by demonstrating that maximizing semantic attribute control encourages minimizing entanglement of latent factors. It shows however that while connected, these two concepts are formally not equivalent. Furthermore, the direction of this relationship, namely that control encourages disentanglement and not the other way around, motivates our choice of using MI rather than $\mathbb{TC}$ in our loss function for the rest of the study.

\section{Application to Anomaly Detection}
\label{pipeline}

Given our lemma and architecture above, the discriminator may learn a representation of the image not only in the auxiliary network but in layers used for discriminating real versus fake. This is due to using shared weights for both the auxiliary network and the discriminator, encouraging the discriminator to learn a mapping that includes potentially semantic information.

Using InfoStyleGAN trained on a subset of each dataset, we then extract the raw vector output by the discriminator at the end of: (a.) the auxiliary network Q, with dimensions depending on how C is implemented (b.) the last convolutional layer in the discriminator network of dimensions 512 by 4 by 4, and (c.) the last dense layer in the discriminator network, typically of dimension 512.

For nomenclature, the acronyms consisting of these different possibilities for representations and discriminative use of GANs for anomaly detection are described in Table~\ref{table:AUCCeleba}. We whiten these embeddings using PCA and keep all components for the best performance. We then use these representations as embeddings for two anomaly detection methods, including one done via one-class support vector machine (OCSVM) and the other via local outlier factor (LOF). One-class support vector machines learn a hypersphere on the data given such that the radius containing most of the data is minimized. LOFs compare a given point to its nearest neighbors, and if the density around the given point is less than those of its neighbors, then it is categorized as an outlier.

As an alternative, we also test using embeddings output from two networks: the average pooling layer from an Inception network pre-trained to classify 1000 categories of objects from Imagenet, and the global representation from the Deep InfoMax (DIM) network. Deep InfoMax \citep{bachman2019amdim} is trained via a contrastive loss to maximize the similarity between two different views of the same image, whilst minimizing the similarity with all other representations from different images. Consequently, it should learn as unique of representation as possible with respect to factors invariant to the differences in the view. Note that as both of these networks were pre-trained on the entirety of ImageNet, they can leverage the extra data to output more unique representations when compared to our methods.

\section{Experiments}

\subsection{Data}
\label{data}
The datasets herein used are the public domain 
CelebA \citep{liu2015faceattributes}, a dataset of celebrity faces over 200,000 images, and Stanford Cars \citep{KrauseStarkDengFei-Fei_3DRR2013}, containing around 16,000 images.

\noindent {\bf CelebA:}
CelebA consists of over 200,000 celebrity faces with various attributes such as gender and age. These images are 218 by 178 pixels, and to preprocess them, we take a 128 by 128 crop with the center (121, 89). 

\noindent {\bf Stanford Cars:}
This dataset contains 16,151 images of vehicles of different makes, models, and years. For preprocessing, we use the same preprocessing that \citep{Karras2019stylegan} used for LSUN Cars on this dataset, targeting a resolution of 512 by 384 pixels. First, we crop height to match the correct aspect ratio, ignoring any images that require upsampling. We then resize to 512 by 384 pixels, and then pad this image to a 512 by 512 image. 

For disentanglement experiments, we use the full dataset in all cases to cover every potential semantic attribute of the dataset.
For anomaly detection experiments, we use the following classes as inliers and outliers: 
male celebrities as inliers and female celebrities as outliers for CelebA, and small vehicles (compacts/sedans) as inliers and large vehicles (trucks/vans) as outliers for Stanford Cars.

\noindent

\subsection{Additional Implementation Details} 
We adopt the basic settings of the implementation from \citep{Karras2019stylegan} for all datasets. However, mixing is turned off, as it was found that mixing the styles reduces disentanglement. Moreover, as we are focused on reconstructing the initial vector used to create the style vector, introducing additional style vectors during generation would cause the information loss term to become ambiguous. For computational efficiency, we append $Q$ to $D$ as in Figure~\ref{fig:TCICSG} similar to \citep{chen2016infogan}. 
For CelebA, we use a $512$ dimension latent code for $(Z', C)$, with $C$ consisting of seven discrete Bernoulli variables, one discrete categorical variable of dimension 3, and ten continuous uniform variables. For Stanford Cars,  we use a $512$ dimension latent code for $(Z', C)$, with $C$ consisting of 20 continuous uniform variables.  For implementing $Q(C | \hat{X})$, we use logits for each of the discrete variables, and treat the posterior distribution of the continuous uniform variables as  Gaussian $\mathcal{N}(\mu (\hat{X}), \sigma^2 (\hat{X}))$ with mean $\mu (\hat{X})$ and variance $\sigma^2 (\hat{X})$. 

Specifically for anomaly detection, we train each GAN on a dataset comprising of solely images in the inlier class and take a subsampling for training OCSVM or LOF. We then test on a balanced test set for both inliers and outliers for each dataset. Each set of representations was normalized, as well as fed through PCA using whitened components. All generative models and Deep InfoMax were found to perform best when all components are used, whereas Inception V3 performs best when 1,024 components were used. These standardized representations were then used to train the OCSVM/LOF model.

\medskip\noindent
\subsection{Disentanglement, Control, and Generative Quality: Quantitative Assessment}

\begin{table}[!t]
    \caption{Quality of generated images: This table demonstrates that adding disentanglement properties does not negatively impact visual performance (FID). Best FID found during training for various architectures including InfoStyleGAN, with standard deviations taken over five estimates of the metric. We also list comparisons to other models on this dataset, with the first three taken from \citep{1711.10337}. As the authors of this paper were conducting a large scale study where models were not fully trained in order to explore hyperparameters, the results for these models could possibly be better.
    }
\centering

    \begin{tabular}{cccc}
    \hline \hline
         Models & CelebA  & Stanford Cars\\
         \hline
        InfoGAN & 9.91 ($\pm 0.06$) 
        &  
        $\dagger$ \\
        StyleGAN & 9.08 ($\pm 0.10$)
        & 
        21.35 ($\pm 2.01$)
        \\
        InfoStyleGAN  & 9.90 ($\pm 0.07$)
        &  
        23.52 ($\pm 1.97$) 
        \\
        InfoStyleGAN-Discrete & 14.3 ($\pm 0.05$)
        &   
        $\dagger$ 
        \\
        \hline
        WGAN GP* & 30.0
        &  
        $\dagger$  \\
        BEGAN* & 38.9
        &  
        $\dagger$  \\
        DRAGAN* & 42.3
        & 
        
        $\dagger$ \\
        COCO-GAN* & 4.2
        & 
         
        $\dagger$ \\
        \hline \hline
        
    \end{tabular}{}
    \label{table:FID}

\end{table}{}

\begin{table}[!t]
    \caption{
    We use the precision and recall metrics from \citep{kynk2019precision} on our trained CelebA models. We use Inception V3 as our feature extractor and either $k=3$ or $k=10$ neighbors.
    }
\centering

    \begin{tabular}{cccccc}
    \hline \hline
         Metrics & InfoGAN  & StyleGAN & InfoStyleGAN & InfoStyleGAN-Discrete\\
         \hline
        Precision $(k=3)$ & 0.654 & 0.720 & 0.629 & 0.635
        \\
        Recall $(k=3)$ & 0.291 & 0.306 & 0.319 & 0.226 
        
        \\
        \hline
        Precision $(k=10)$ & 0.855 & 0.893 & 0.848 & 0.850
        \\
        Recall $(k=10)$ & 0.498 & 0.561 & 0.562 & 0.447 
        
        \\
        
        \hline \hline
        
    \end{tabular}{}
    \label{table:pr3}

\end{table}{}
Although several metrics have emerged to characterize quality and diversity in generative models, we use FID \citep{heusel2017gans} as our primary measure of distance between the two distributions, which is computed as:
\begin{align}
 FID = \| \mu_r - \mu_g \|^2 + \text{Tr}(\Sigma_r + \Sigma_g - 2(\Sigma_r\Sigma_g)^{0.5})
\end{align}
where $\| \cdot \|^2$ and $\text{Tr}(\cdot)$ denote the $L_2$ norm and the trace, respectively, and the assumption is made that $X_r \sim \mathcal{N}(\mu_r, \Sigma_r)$, i.e., a normal distribution (with mean-vector $\mu_r$ and covariance matrix $\Sigma_r$) for the activations from the Inception v3 pool layer for real examples, and likewise $X_g \sim \mathcal{N}(\mu_g, \Sigma_g)$ for generated examples.

FID is reported for each of the architectural variants, and both datasets, as shown in Table \ref{table:FID}. Additionally comparisons are shown in the last rows and all but COCO-GAN are taken from~\citep{1711.10337}, where the training scheme is different; that study used a smaller architecture and a $64\times64$ resolution version of CelebA. For COCO-GAN, \citep{lin2019cocogan} used a larger model along with a different resolution of CelebA. InfoGAN and InfoStyleGAN perform slightly worse than StyleGAN in terms of FID, and using only discrete variables in C in InfoStyleGAN performs significantly worse.

As the FID comprises of two terms that can roughly be taken to mean the fidelity and diversity of the fake distribution, we use two additional metrics, precision and recall from \citep{kynk2019precision}, to better characterize the fidelity (precision) and diversity (recall) separately for CelebA. \citep{kynk2019precision} uses the image features of a specific network, Inception V3, chosen in our case to match the feature space of FID, to construct an approximate manifold of each the real images and fake images separately. This manifold is comprised of hyperspheres around each individual point in the dataset, where each hypersphere has the smallest radius that contains the nearest $k$ neighbors. The precision is then defined as the proportion of fake images whose features lie in the constructed manifold of real images, and the recall is defined as the proportion of real images whose features lie in the constructed manifold of fake images.

We show the results in Table \ref{table:pr3}, where we see that StyleGAN has the highest precision at 0.720, while InfoStyleGAN has the highest recall at 0.319, for $k=3$. Comparing InfoGAN and InfoStyleGAN, we see that the results appear to match both models having nearly equal FID, with InfoGAN having marginally higher precision and lower recall than InfoStyleGAN, exhibiting some trade-off between the quality and diversity. For $k=10$, we see that the metrics overall follow a similar trend as what is presented in \citep{kynk2019precision} for the FFHQ dataset, which can be construed as a higher quality CelebA.

In order to characterize the ability of the algorithms to control  individual known image attributes of generated images, we also utilize the mutual information gap (MIG) \citep{chen2018isolating} for CelebA. Stanford Cars did not have ground truth attributes associated with it.

\noindent {\em Mutual Information Gap:} As we have ground truth attribute labels for CelebA describing various semantic attributes $\{V_k\}_{k=1}^K$, where $K$ is the number of ground truth attributes, we can use those directly in a supervised fashion to estimate the mutual information between each $V_k$ and each $C_i$, $i=1,\ldots,L$, as described by the auxiliary network $\mathcal{Q}$. Consequently, we estimate 
\begin{align}
    \mathbb{I}(V_k; C_i) &= \mathbb{E}_{V_k, C_i' \sim \mathcal{Q}(C_i | V_k)} \left(\log \sum_{\Tilde{x} \in X_{V_k}} \mathcal{Q}(C_i' | \Tilde{x}) p(\Tilde{x} | V_k)\right) + \mathbb{H}(C_i) 
\end{align}
over k and i, where $X_{v_k}$ is the set of images that correspond to having the label $v_k$. As the overall conditional probability $p(\Tilde{x} | v_k)$ is unknown, we assume a uniform distribution over all $\Tilde{x}$ that have $v_k$ as a label. The MIG is then 
 \begin{align}
     MIG &= \frac{1}{K} \sum_{k=1}^{K} \frac{1}{\mathbb{H}(V_k)} \left(\mathbb{I}(V_k; C_{i(k)}) - \max_{i \neq i(k)} \mathbb{I}(V_k; C_{i})\right) 
 \end{align}
 where $i(k) = \text{argmax}_i \mathbb{I}(V_k; C_i)$
is the index over the latent factors $C_i$ that selects the $C_i$ with the maximum mutual information with respect to the given ground truth attribute. Consequently, the MIG is the normalized difference between the maximum and second largest mutual information.
For CelebA, the set of attributes $\{V_k\}$ consists of 40 attributes, including binary variables for smiling, attractiveness, etc. 

There are two consequences of using this measure: (1) a larger value indicates that the information about a ground truth attribute is concentrated in a single latent factor which aligns with the goals of disentanglement and (2) a small value does not necessarily indicate that our model does not successfully disentangle but that the semantic attributes it discovers may not align with any of the ground truth semantic attributes. Consequently, we also include the maximum mutual information to indicate if this is occurring.
In Table \ref{table:MIG}, we see that InfoStyleGAN for CelebA does improve upon disentanglement compared to InfoGAN.

\begin{table}[!htbp]
    \caption{Faculty for attribute control and disentanglement: For CelebA, we use Mutual Information Gap (MIG, Equation (13)) on the ground truth binary labels of CelebA using the auxiliary network. Unlike the various autoencoder architectures, we do not necessarily expect that the semantic attributes learnt will be all-inclusive; so we include the maximum mutual information (third column) over each ground truth attribute and $C_i$ pair to see if any are actually significant. We do see an improvement in disentanglement using architectures with mutual information loss, whereas InfoStyleGAN-Discrete discovers semantic attributes that are not aligned with the ground truth attributes. Consequently, for the discrete-only architecture, the MIG measure is inconclusive.\\}

\centering
    \begin{tabular}{ccc}
    \hline \hline
         Models & MIG  & $\max_{k,i} \mathbb{I}(V_k; C_i)$\\
         \hline
        InfoGAN & 2.4e-2 
        & 4.9e-2 \\
        InfoStyleGAN & 3.4e-2 
        & 5.9e-2
        \\
        InfoStyleGAN-Discrete &  1.2e-4 
        & 3.0e-4 
        \\
        \hline \hline
        
    \end{tabular}{}
    \label{table:MIG}

\end{table}{}

\subsection{Disentanglement, Control, and Generative Quality: Qualitative Assessment}
\begin{figure}[!t]
\begin{subfigure}{0.48\textwidth}

\includegraphics[width=\linewidth]{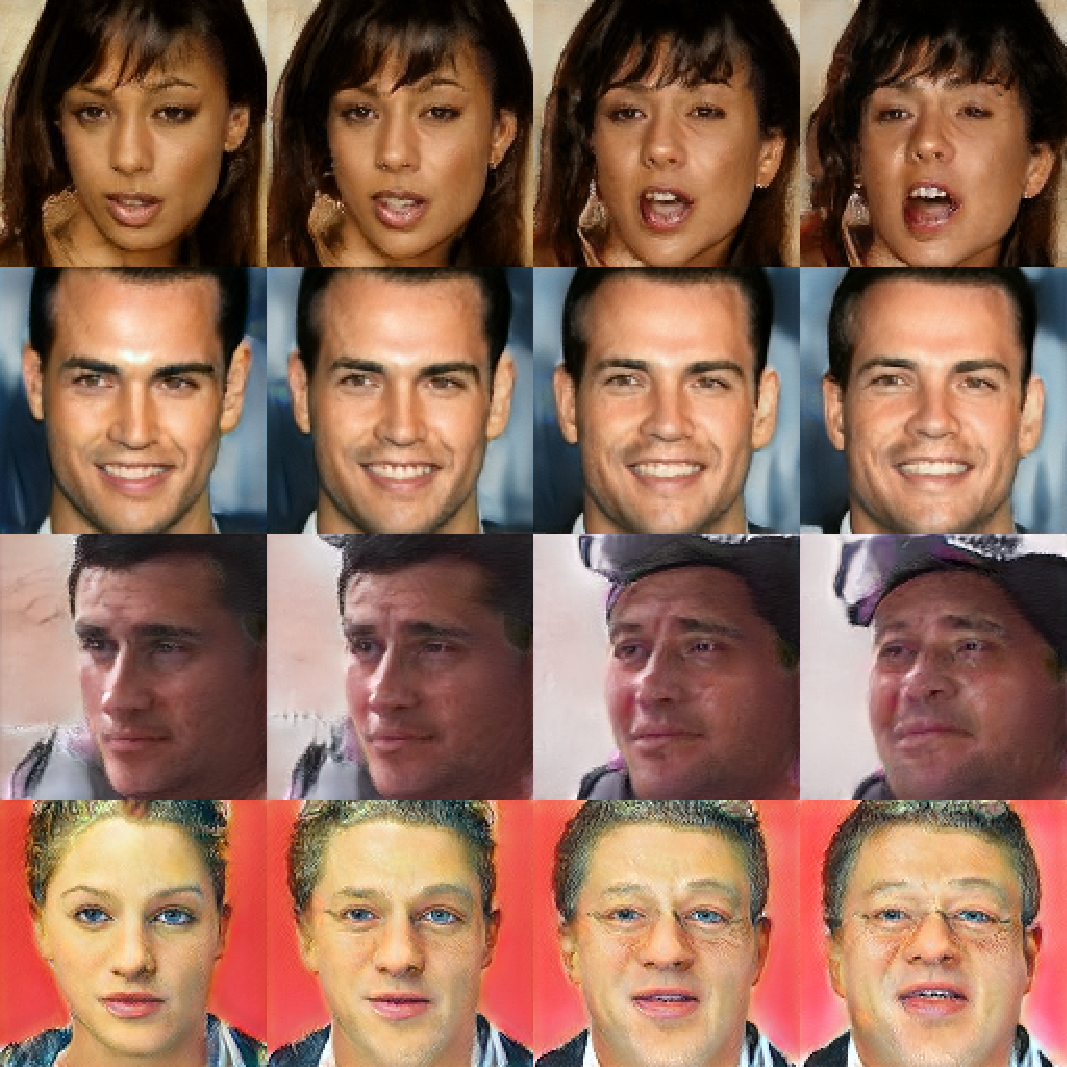}
\caption{InfoStyleGAN: Continuous variable controlling of head orientation up and down.}\label{azimuth}
\end{subfigure}
\hfill
\begin{subfigure}{0.48\textwidth}

\includegraphics[width=\linewidth]{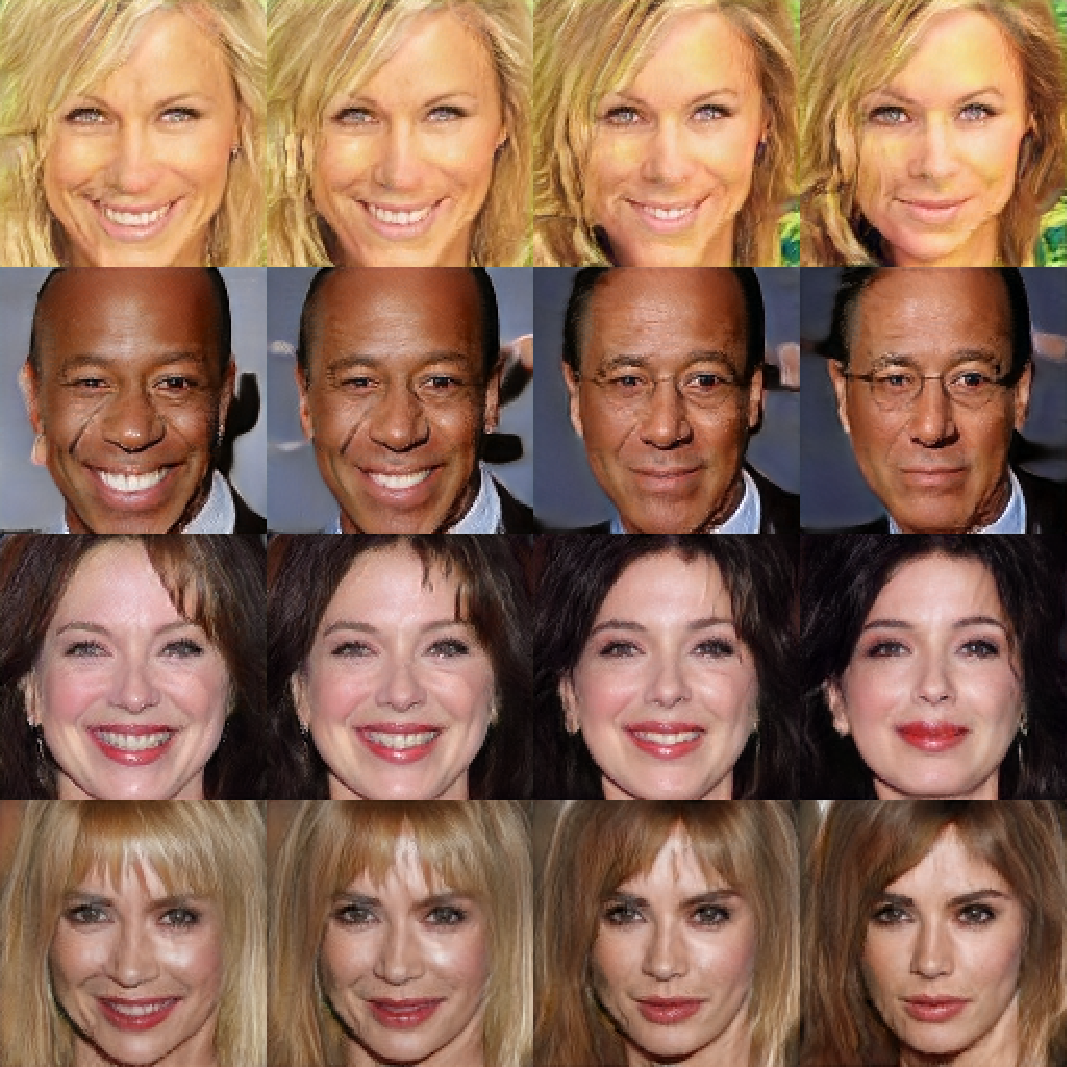}
\caption{InfoStyleGAN: Continuous variable controlling smiling and hair color.}\label{smiling-hair}
\end{subfigure}
\begin{subfigure}{0.48\textwidth}

\includegraphics[width=\linewidth]{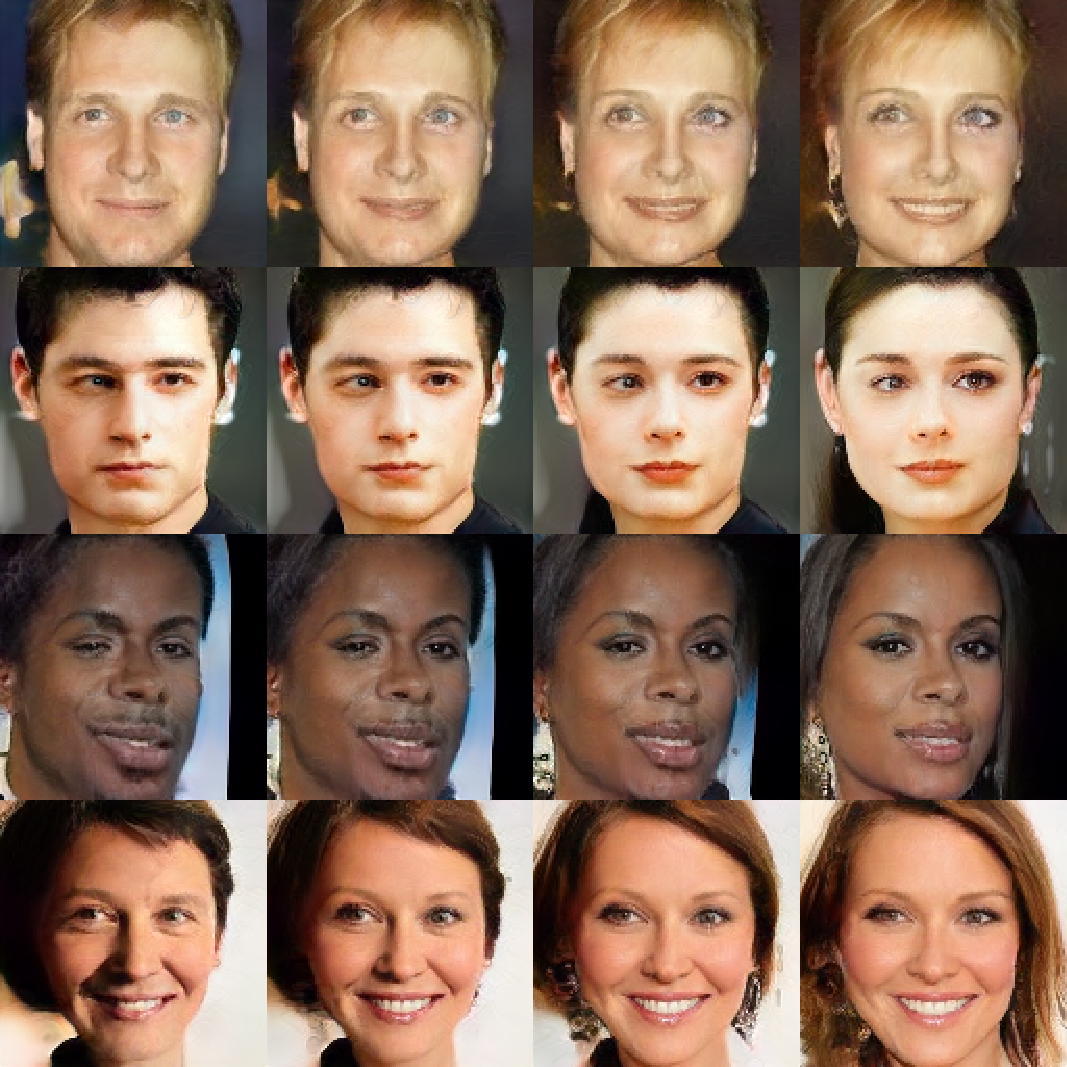}
\caption{InfoGAN: Continuous variable controlling the rotation side to side.}\label{fig:baseline-gender}
\end{subfigure}
\hfill
\begin{subfigure}{0.48\textwidth}
\includegraphics[width=\linewidth]{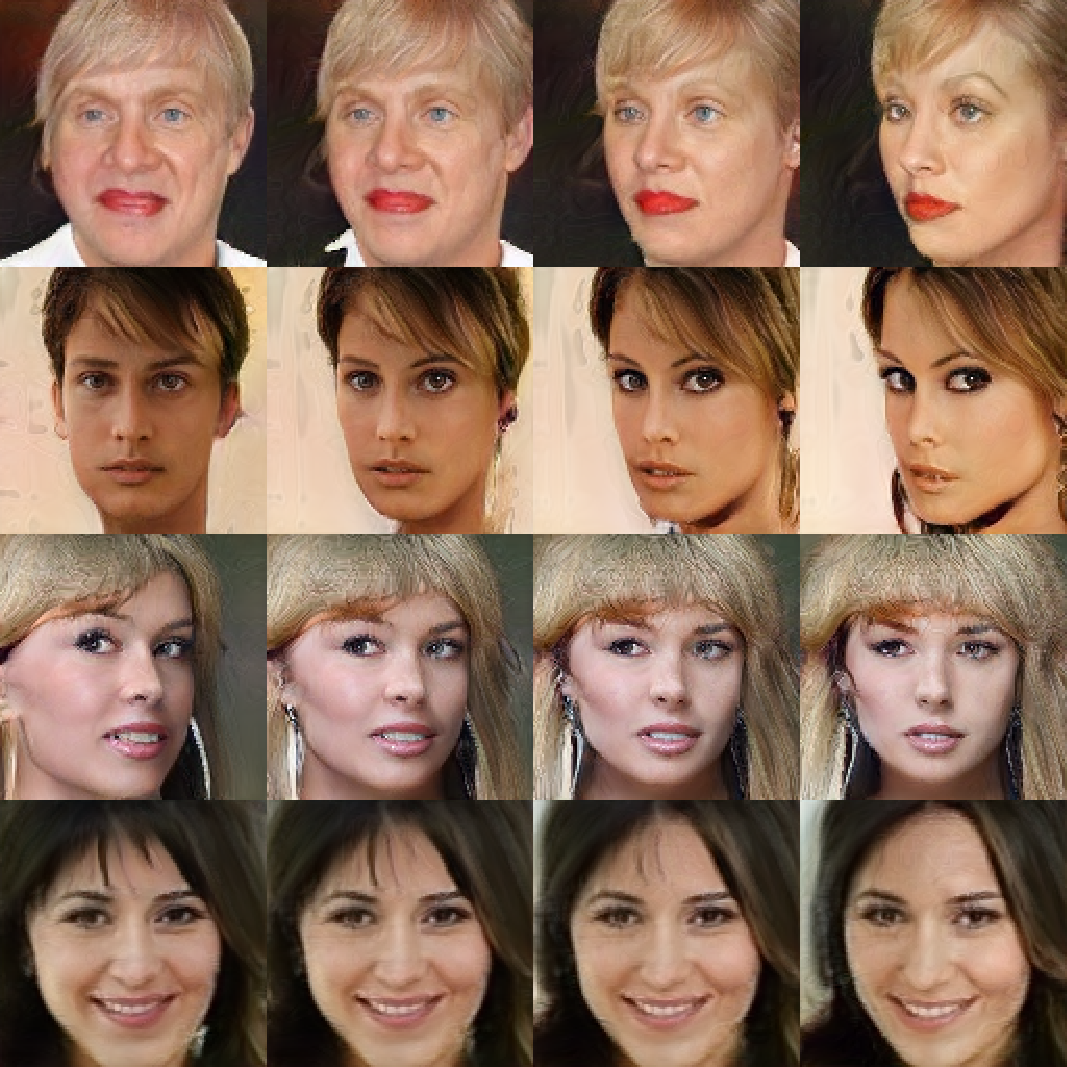}
\caption{InfoGAN: Continuous variable controlling gender.}\label{contradiction}
\end{subfigure}

\caption{Effects of control and disentanglement by InfoGAN and InfoStyleGAN on CelebA.}

\vspace{-0.2cm}
\end{figure}
Figures \ref{azimuth} and \ref{smiling-hair} display example attribute control for the proposed architecture for CelebA. Similar to the original results for InfoGAN, we see variables corresponding to emotion and head position; however these correspond to a continuous variable here rather than a categorical variable and exhibit in our case a greater magnitude of control. Surprisingly, the image fourth from the right in Figure \ref{azimuth} appears to also control the glare of the glasses as the head is tilted up. Although some entanglement is still seen, such as the example on the left side of Figure \ref{azimuth}, where the faces become more masculine as the orientation increases, or the third from the right on Figure \ref{smiling-hair} which has glasses appear through increasing the smile. However, across all source images, the effect the attribute control is always very consistent in nature.

In contrast, for InfoGAN, we see comparatively more entanglement across the various source images and less control. For Figure \ref{fig:baseline-gender}, we see that InfoGAN exhibits a significant degree of control over the horizontal orientation of the head, but the endpoints are not consistent with their starting and final positions compared to Figure \ref{azimuth}. For Figure \ref{contradiction}, we also see control over the gender, which is entangled somewhat with smiling, and the degree of this smiling is not as significant as Figure \ref{smiling-hair}. Consequently, these controls on the tuned baseline do not show a consistent effect like those on InfoStyleGAN. The continuous variables in both cases exhibited the most interesting factors, whereas the Bernoulli or categorical variables primarily affected the pose.

\subsection{Anomaly Detection: Quantitative Assessment}

\begin{table}[!ht]
\centering
    \caption{Table describing ROC AUC, overall accuracy, and F1 score for each method tested on CelebA dataset with confidence intervals in brackets. The first entry in each method is the network used to get the representation, the second entry for generative methods is the representation used, and the last is the anomaly detection method. The second column is the abbreviation for each method, taking the first letter from each field. Inliers are male celebrities, and outliers are female celebrities.\\}

    \resizebox{\textwidth}{!}{
    \begin{tabular}{lllll}
    \hline \hline
         Method &  & ROC AUC & Accuracy & F1 Score \\
         \hline
        InfoStyleGAN $\rightarrow$ Q $\rightarrow$ OCSVM & IQO & 0.567 [0.559, 0.576]& 54.47 \% [53.74\%, 55.19\%] & 0.524\\
        InfoStyleGAN $\rightarrow$ Q $\rightarrow$  LOF & IQL  & 0.567 [0.559, 0.576]& 52.47 \% [51.74\%, 53.2\%] & 0.244\\
        InfoStyleGAN $\rightarrow$ Conv $\rightarrow$  OCSVM & ICO & 0.593 [0.585, 0.601]& 58.09 \% [57.37\%, 58.81\%] & 0.632\\
        InfoStyleGAN $\rightarrow$ Conv $\rightarrow$LOF & ICL & 0.6 [0.592, 0.608]& 52.12 \% [51.39\%, 52.85\%] & 0.329\\
        InfoStyleGAN $\rightarrow$ Dense $\rightarrow$  OCSVM & IDO & 0.607 [0.599, 0.615]& {\bf 58.83 \% [58.11\%, 59.55\%]} & {\bf 0.643}\\
        InfoStyleGAN $\rightarrow$ Dense $\rightarrow$  LOF & IDL & {\bf 0.608 [0.6, 0.617]} & 54.24 \% [53.52\%, 54.97\%] & 0.392\\
        \hline \hline
        Inception V3 $\rightarrow$  OCSVM & IO & 0.629 [0.621, 0.638]& 59.24 \% [58.52\%, 59.96\%] & 0.704\\
        Inception V3 $\rightarrow$  LOF & IL & 0.629 [0.621, 0.637]& 60.96 \% [60.25\%, 61.67\%] & 0.661\\
        Deep InfoMax $\rightarrow$  OCSVM & DO & 0.604 [0.595, 0.612]& 51.78 \% [51.05\%, 52.51\%] & 0.675\\
        Deep InfoMax $\rightarrow$  LOF & DL & 0.603 [0.595, 0.611]& 57.17 \% [56.44\%, 57.89\%] & 0.696\\
         \hline
        
    \end{tabular}{}
    }

    \label{table:AUCCeleba}
\end{table}{}
\begin{table}[!ht]
    \caption{Table describing ROC AUC, overall accuracy, and F1 score for each method tested on the Stanford Cars dataset with confidence intervals in brackets. Inliers are small vehicles such as sedans and compacts, and outliers are large vehicles such as trucks and vans.}

\begin{center}
    \resizebox{\textwidth}{!}{
    \begin{tabular}{lllll}
    \hline \hline
         Method &  & ROC AUC & Accuracy & F1 Score \\ 
         \hline
        InfoStyleGAN $\rightarrow$ Q $\rightarrow$ OCSVM & IQO & 0.326 [0.293, 0.359]& 39.10 \% [36.08\%, 42.12\%] & 0.206 \\ 
        InfoStyleGAN $\rightarrow$ Q $\rightarrow$  LOF & IQL  & 0.346 [0.312, 0.38]& 44.50 \% [41.42\%, 47.58\%] & 0.031 \\
        InfoStyleGAN $\rightarrow$ Conv $\rightarrow$  OCSVM & ICO & 0.379 [0.345, 0.414]& 42.60 \% [39.54\%, 45.66\%] & 0.305 \\
        InfoStyleGAN $\rightarrow$ Conv $\rightarrow$LOF & ICL & 0.460 [0.424, 0.496]& 48.70 \% [45.60\%, 51.80\%] & 0.111 \\
        InfoStyleGAN $\rightarrow$ Dense $\rightarrow$  OCSVM & IDO & 0.809 [0.782, 0.836]& 71.70 \% [68.91\%, 74.49\%] & {\bf 0.748} \\ 
        InfoStyleGAN $\rightarrow$ Dense $\rightarrow$  LOF & IDL & {\bf 0.867 [0.845, 0.89]}& {\bf 77.50 \% [74.91\%, 80.09\%]} & 0.742 \\ 
        \hline \hline
        Inception V3 $\rightarrow$  OCSVM & IO & 0.807 [0.78, 0.834]& 74.40 \% [71.70\%, 77.10\%] & 0.767 \\ 
        Inception V3 $\rightarrow$  LOF & IL & 0.707 [0.675, 0.739]& 52.70 \% [49.61\%, 55.79\%] & 0.218 \\ 
        Deep InfoMax $\rightarrow$  OCSVM & DO & 0.577 [0.542, 0.613]& 55.80 \% [52.72\%, 58.88\%] & 0.546 \\ 
        Deep InfoMax $\rightarrow$  LOF & DL & 0.569 [0.533, 0.604]& 52.40 \% [49.30\%, 55.50\%] & 0.244 \\ 
         \hline
        
    \end{tabular}{}
    }
\end{center}

    \label{table:AUCCars}
\end{table}{}

Tables \ref{table:AUCCeleba} and \ref{table:AUCCars} report comparisons between the various different variants of our pipeline for anomaly detection on CelebA and Stanford Cars respectively. We weigh ROC AUC most heavily, though accuracy and F1 score are closely correlated.

For CelebA in Table~\ref{table:AUCCeleba}, we see that the Q network representations performs the worst, followed by the convolutional representations, and that the dense representations are the best.

However, the generative models perform on par overall with the discriminative methods, performing better than Deep InfoMax and worse than Inception V3.

For Stanford Cars in Table~\ref{table:AUCCars}, we also note that the dense representations perform best, while all other representations perform significantly worse. Given the failure of the Q network to optimize for its lower bound, the performance of the Q representations is somewhat expected, though not for the other two representations. One possible reason for this disparity can be deduced from Figure \ref{carmatrix} (discussed below) since large and small vehicles appear similar locally, and the global structure of the dense representation helps to discern the two classes. Interestingly, the generative approach actually does overtake both discriminative approaches here, which also uses a more global representation of the image.

Finally, we also plot the ROC curves for each dataset in Figures \ref{ROCCurvesFaces} and \ref{ROCCurvesCars}. For Stanford Cars, we see that improvements in AUC are distributed evenly throughout the curve. However, for CelebA, we see that we do not get any actual improvement in the true positive rate between different methods without increasing the false positive rate significantly. This may be due to how gender is characterized by finer details, such as makeup or facial structure, which may be treated as invariants by only training on the inliers.

\begin{figure}[h!]
\begin{subfigure}{0.5\textwidth}
\includegraphics[width=\linewidth]{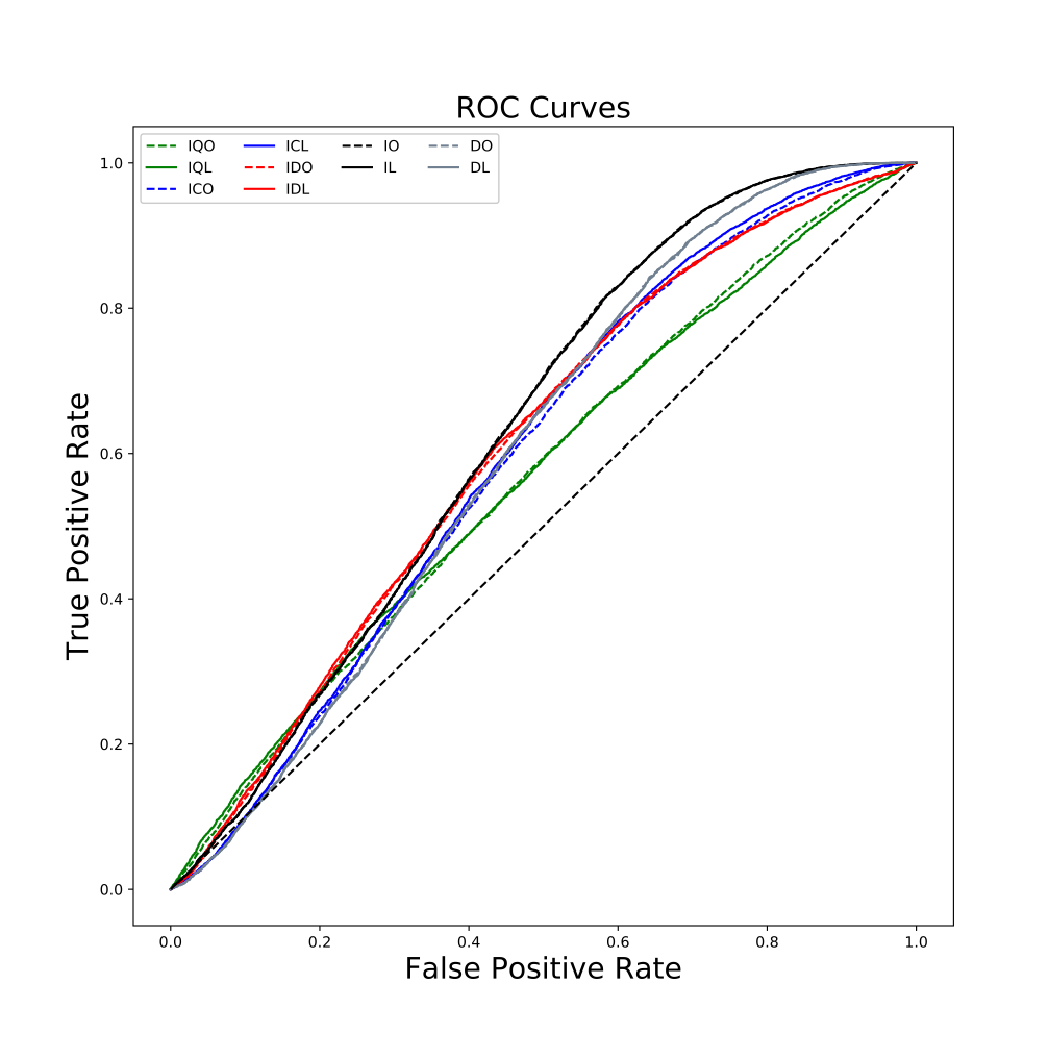}

\caption{CelebA: Table \ref{table:AUCCeleba}. }
\label{ROCCurvesFaces}
\end{subfigure}
\begin{subfigure}{0.5\textwidth}
\includegraphics[width=\linewidth]{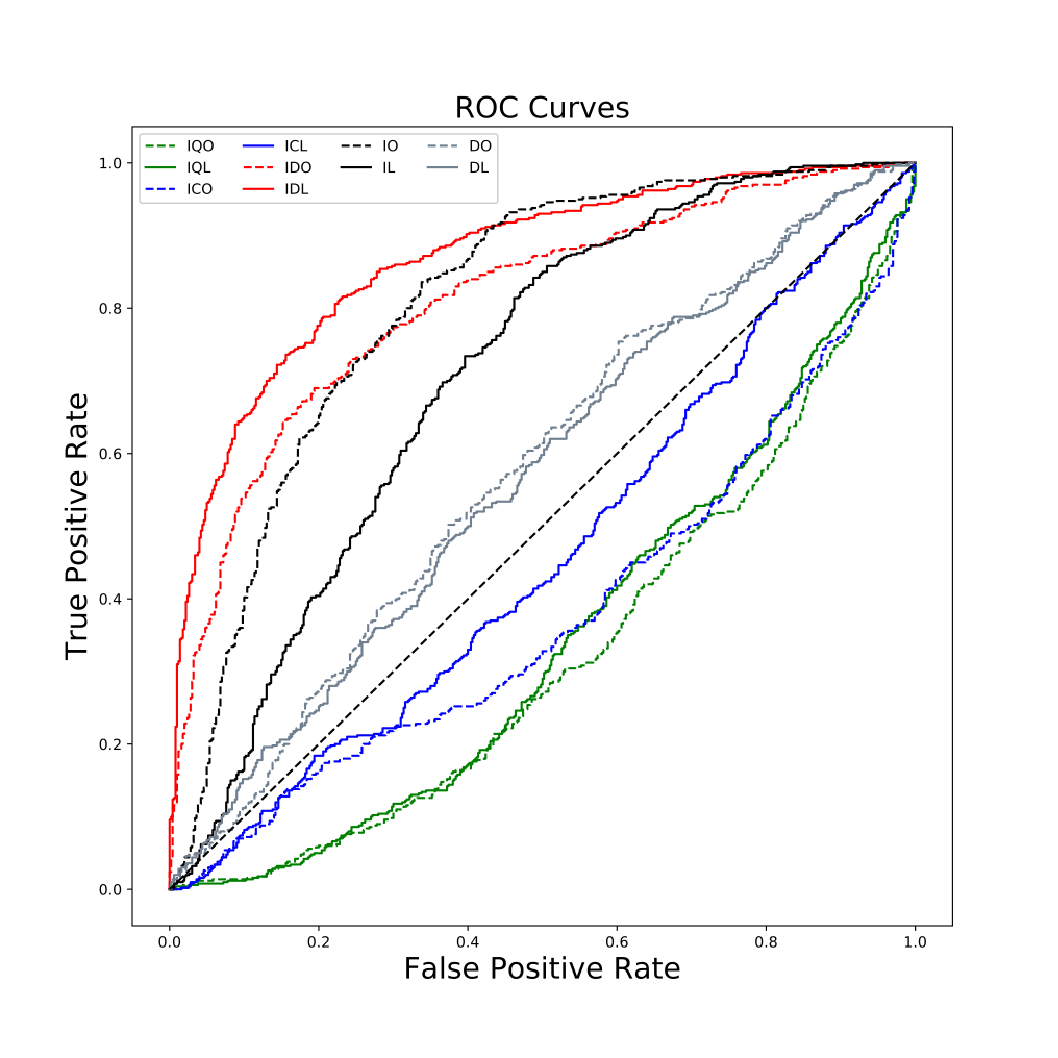}

\caption{Stanford Cars: Table \ref{table:AUCCars}.}
\label{ROCCurvesCars}
\end{subfigure}

\caption{Plots of ROC curves for each dataset used. Dashed lines are methods using OCSVM, while solid lines use LOF.}

\end{figure}
 
\medskip\noindent
 
\subsection{Anomaly Detection: Qualitative Assessment}

\begin{figure}[h!]

\begin{subfigure}[h!]{0.5\linewidth}
\includegraphics[width=\linewidth]{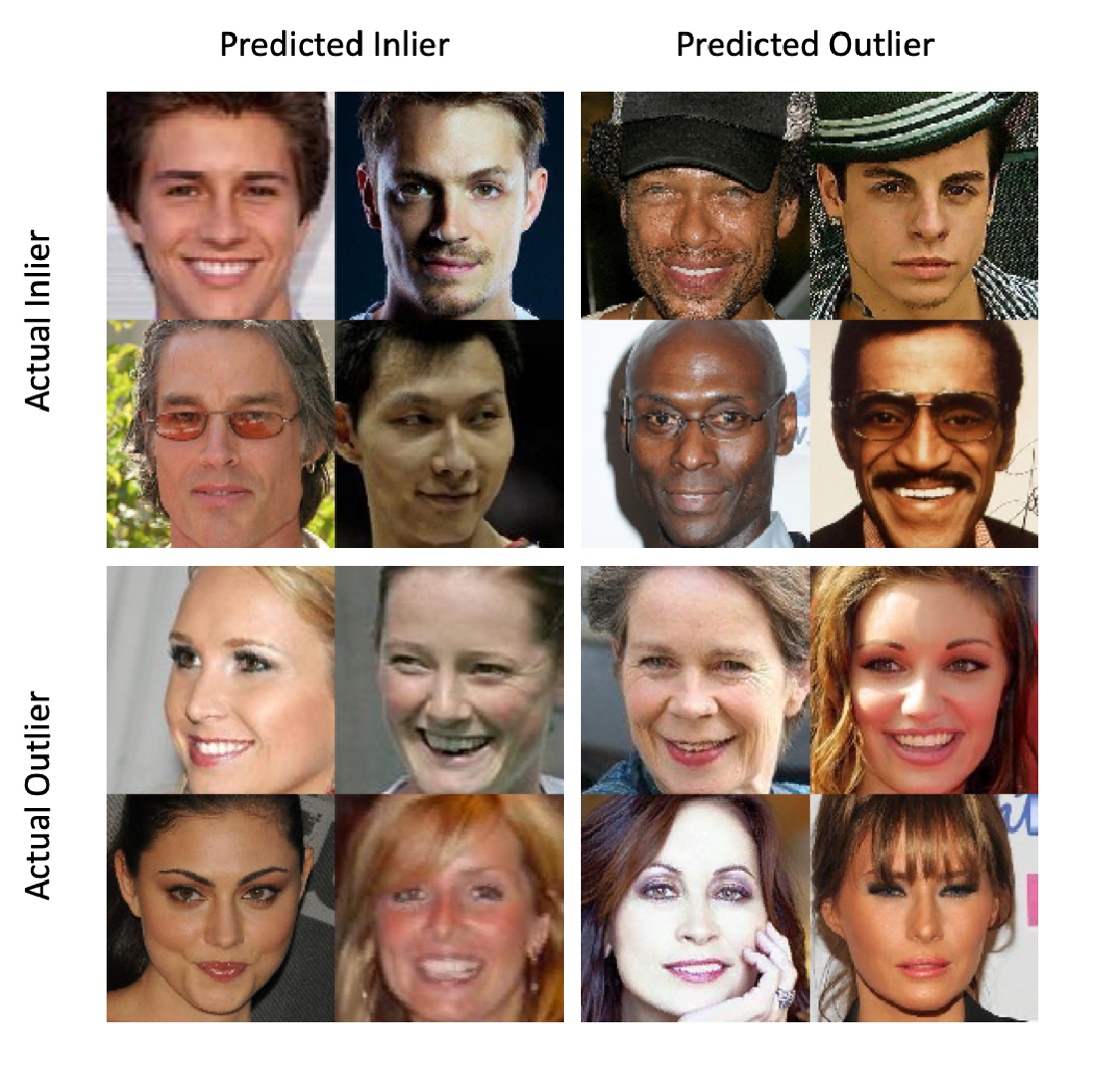}
\caption{CelebA}
\label{facematrix}
\end{subfigure}
\begin{subfigure}[h!]{0.5\linewidth}
\includegraphics[width=\linewidth]{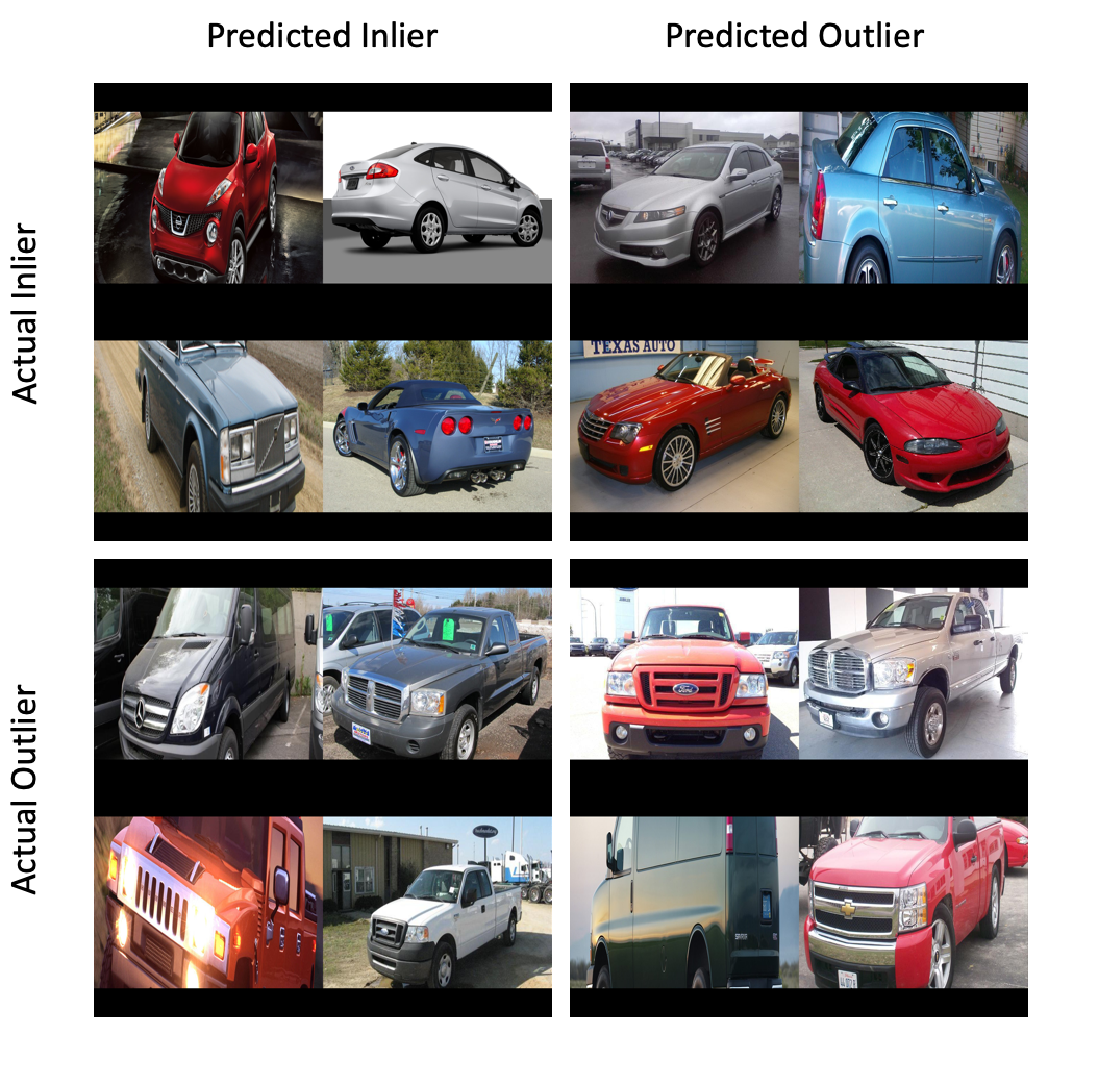}
\caption{Stanford Cars}
\label{carmatrix}
\end{subfigure}
\caption{A confusion matrix, with each column denoting the predictions and each row the ground truth, of the anomaly detector using the dense representation of InfoStyleGAN with OCSVM on each dataset.}
\end{figure}

Finally, we show in Figures \ref{facematrix} and \ref{carmatrix} a confusion matrix containing example images for our anomaly detector for 
CelebA and Stanford Cars respectively. Each figure is split into four quadrants, with the columns denoting the predictions by the model and rows denoting the true status of the image.

\section{Discussion}
\subsection{Control}

Visual inspection of our results demonstrates control of semantic variables as well as a good degree of disentanglement for our proposed full fledged model for 
CelebA. 
In particular: we found on par FID, i.e., 9.90 for InfoStyleGAN when compared to the ablated baselines (i.e., 9.91 for InfoGAN), while InfoGAN did not perform as well with regard to disentanglement (i.e., MIG=3.4e-2 for InfoStyleGAN, compared to MIG=2.4e-2 for InfoStyleGAN). FID was worse for other state of the art algorithms that did not attempt to disentangle, except for COCO-GAN, which uses a much bigger model, so the comparison is not really apple-to-apple. 

In general, comparing InfoGAN to our method InfoStyleGAN, we see that InfoGAN does not have the same level of control as InfoStyleGAN likely due to these architectural differences despite having a comparable distance to the real distribution. From the purported benefits of StyleGAN, this disentanglement on the semantic variables is consistent with increased disentanglement of the overall latent space. Also interesting is that the overall distance from the real distribution is similar for InfoGAN and InfoStyleGAN, with StyleGAN performing slightly better. As the FID comprises of two terms, one measuring the distance between the two means and the other measuring the distance between the covariances. As the mutual information maximization increases the entropy of the generated images, then presumably the info-methods are closer in terms of diversity rather than closer in terms of the average, especially given that the generators still have to generate all combinations of semantic attributes and still maintain the image's realism. These generators may not be able to infer from the training dataset how to combine these factors in a realistic way, which may be reflected in the precision and recall metrics in Table \ref{table:pr3} for $k=3$.

Of note, during our experiments we observed that CelebA did not have $L_{info}$ fully maximized, which explains why some entanglement is still observable in some of the images. 
Consequently an avenue of future work is to address this issue by trying to explicitly minimize total correlation in the loss function along with attempting different methods for the maximization of mutual information such as those found in \citep{poole2019bounds}.

In terms of Stanford Cars, the best lower bound of the mutual information it could achieve was roughly 0.1. One possibility for this performance may be that due to the dataset having significantly less data compared to CelebA, the generator does not know how to combine potential attributes in a realistic way, so it focuses more on maintaining its realism rather than optimizing for the mutual information, which may also relate to its performance in anomaly detection.  

In aggregate, we find that the proposed approach gave promising results for control and discovery of semantic attributes. Our findings that sometimes one latent factor appeared to influence jointly two different attributes is consistent with our lemma explaining that it is the control objective (via mutual information maximization) that encourages disentanglement (via total correlation minimization). Indeed since mutual information maximization is achieved via maximization of a lower bound on mutual information, if this bound is not tight, then minimization of total correlation down to zero may not be achieved (in the case of discrete variables) which explains residual correlation between factors. 

We believe that the discovery of latent factors that control semantic attributes has other potential applications for AI tasks such as debiasing in healthcare. It allows the discovery of factors of variations that can be the basis for a sensitivity analysis that can help address issues of bias in AI. These can be of importance for healthcare applications where phenotype discovery and sensitivity analysis of AI with regard to these phenotype is of interest. Fundamentally it is complementary to methods that address direct classification of image semantic attributes when those attributes are known a priori. When attributes and phenotypes are not known, the problem is much more arduous and our method can help address this issue.

\subsection{Anomaly Detection}

From the above results in the experiment section with regard to anomaly detection, we remark that anomaly detection based on representation learning from these generative models provides encouraging results. We view ROC AUC as our primary metric for comparison, though Accuracy and F1 Score are also both correlated with ROC AUC. 

In general, CelebA was a difficult dataset for all methods, with the best overall, Inception V3, only achieving an AUC of 0.629, a performance close to InfoStyleGAN with LOF, which had AUC of 0.608. Also, as can be seen in Figure \ref{ROCCurvesFaces}, for regimes with low false positive rates, InfoStyleGAN based methods somewhat outperform the discriminative  methods. From the off diagonals of Figure \ref{facematrix}, one possible reason for this could be that the most apparent features of gender could actually apply to both inliers and outliers. As an example, take the top rightmost image in the predicted outlier/actual inlier and the bottom leftmost image in the predicted inlier/actual outlier. Ignoring the hat, both have thick eyebrows and darkened eyes, with the women having slightly redder lips and a slightly different facial structure.

For Stanford Cars, the generative models actually improve on the discriminative methods, which is fairly surprising given that large vehicles and small vehicles are actually contained in several separate categories in ImageNet. One possible reason is that of domain shift, given some of the images in Figure \ref{carmatrix} are heavily zoomed in, making distinguishing between small and large vehicles difficult from purely the size and forcing features to focus on areas such as the grill and body shape. Similar reasons possibly explain why the dense/global representation performs so much better than the other representations.

\section{Conclusion}

This study is concerned with generative models that address the joint objectives of 
discovery, disentanglement, and control of image semantic attributes, as well as the goal of performing anomaly detection using representation from this model. We describe a method that uses multiscale generative models and maximize mutual information (so called InfoStyleGAN) and achieve those joint goals -- as evaluated both quantitatively and qualitatively. Results show that our method is competitive in two datasets in terms of anomaly detection compared with models trained on significantly larger datasets with multiple diverse classes, giving promising new directions to continue research into generative anomaly detection methods.

\end{document}